\documentclass[letterpaper, 10 pt]{article}
\usepackage[T1]{fontenc}
\usepackage[utf8]{inputenc}
\usepackage{amsmath}
\usepackage{amsthm}
\usepackage{amsfonts}
\usepackage{amssymb}
\usepackage{graphicx}
\usepackage{pifont}
\usepackage{mathtools}
\usepackage{appendix}
\usepackage{hhline}
\usepackage{times}
\usepackage{algorithm}
\usepackage{algorithmic}
\usepackage{multirow}
\usepackage{multicol}
\usepackage{bm}
\usepackage{epsfig}
\usepackage{epstopdf}
\usepackage{color}
\usepackage{soul}
\usepackage{mathtools}
\usepackage{bbm}
\usepackage{tikz}

\usepackage[margin=1 in]{geometry}

\makeatletter
\newcommand\footnoteref[1]{\protected@xdef\@thefnmark{\ref{#1}}\@footnotemark}
\makeatother

\usepackage[colorlinks=true, linkcolor=blue, anchorcolor=blue, citecolor=blue, bookmarks=false]{hyperref}
\usepackage{enumitem}

\DeclareMathOperator*{\argmax}{arg\,max}

\usepackage{empheq}
\usepackage{lipsum}

\usepackage{titlesec}
\titleformat*{\section}{\large\bfseries}
\titleformat*{\subsection}{\large\bfseries}
\titleformat*{\subsubsection}{\bfseries}
\titleformat*{\paragraph}{\bfseries}
\titleformat*{\subparagraph}{\bfseries}

\numberwithin{equation}{section}
\allowdisplaybreaks

\theoremstyle{plain}
\newtheorem{theorem}{Theorem}[section]
\newtheorem{lemma}{Lemma}[section]

\newtheorem{proposition}{Proposition}[section]

\theoremstyle{definition}
\newtheorem{definition}{Definition}[section]

\theoremstyle{remark}

\usepackage{cancel}
\usepackage{blindtext}
\newcommand{\E}{\mathbb{E}}

\title{\LARGE \bf On the Linear convergence of Natural Policy Gradient Algorithm\footnote{A version of this paper was first submitted to a conference in Mar 2021.}}

\author{
{\normalsize Sajad Khodadadian}\footnote{H. Milton Stewart School of Industrial \& Systems Engineering, Georgia Institute of Technology, Atlanta, GA, 30332, USA, {\tt\small \{\href{mailto:skhodadadian3@gatech.edu}{skhodadadian3}, \href{mailto:sushil@gatech.edu}{sushil}, \href{mailto:prakirt@gatech.edu}{prakirt}, \href{mailto:siva.theja@gatech.edu}{siva.theja}\}@gatech.edu}} $\,$\and
{\normalsize Prakirt Raj Jhunjhunwala}$^\dagger$\and
{\normalsize Sushil Mahavir Varma}$^\dagger$\and
{\normalsize Siva Theja Maguluri}$^\dagger$
}
\date{}

\begin{document}

\maketitle

\setlength{\abovedisplayskip}{5pt}
\setlength{\belowdisplayskip}{5pt}

\begin{abstract}
Markov Decision Processes are classically solved using Value Iteration and Policy Iteration algorithms. Recent interest in Reinforcement Learning has motivated the study of methods inspired by optimization, such as gradient ascent. Among these, a popular algorithm is the Natural Policy Gradient, which is a mirror descent variant for MDPs. This algorithm forms the basis of several popular Reinforcement Learning algorithms such as Natural actor-critic, TRPO, PPO, etc, and so is being studied with growing interest. It has been shown that Natural Policy Gradient with constant step size converges with a sublinear rate of $\mathcal{O}(1/k)$ to the global optimal. In this paper, we present improved finite time convergence bounds, and show that this algorithm has \textit{geometric} (also known as linear) asymptotic convergence rate. We further improve this convergence result by introducing a variant of Natural Policy Gradient with adaptive step sizes. Finally, we compare different variants of policy gradient methods experimentally.
\end{abstract}

\section{Introduction}
Markov Decision Process (MDP) \cite{bellman1954theory} is a general framework to model sequential decision-making systems. In an MDP, in each time step, an agent chooses a control (action) based on a policy in hand and depending on its action the agent receives some reward. The goal of the agent is to find a policy that maximizes its long-term reward. Mathematical modeling of the systems with MDPs has numerous applications \cite{white1985real}. An MDP forms the underlying model to study Reinforcement Learning (RL) algorithms, and this has brought renewed interest in their study.  

The classical approach to finding the optimal policy of an MDP is using Dynamic Programming to solve the Bellman fixed point equation. Value Iteration (VI) and policy iteration (PI) are the two most popular approaches in this paradigm. VI, is an iterative method to find the optimal policy, which exploits the contraction property of MDPs \cite{puterman1995markov}. By application of the Banach fixed point theorem, it is known that VI can find the optimal policy geometrically fast.

In contrast to VI, PI works in the space of policies, where the algorithm searches for the optimal policy by iteratively improving the current policy. It is known that in a finite MDP, PI can find the optimal policy in a finite number of steps \cite{puterman1995markov}. However, the main drawback of PI is the sharp movement of the algorithm on the space of the policies (usually deterministic). This makes the algorithm highly unstable, and in a sample-based setting (such as Reinforcement Learning), PI can result in a high variance in the estimation of the optimal policy.

Policy Gradient methods are ``smoother'' variants of PI that use gradient-descent-like optimization methods in the space of \textit{stochastic} policies, to perform policy improvement \cite{sutton1999policy}. Here, the policy is viewed as a variable, and the goal is to maximize the objective function in this variable, which is simply the long-term reward achieved by the policy. PG methods in Reinforcement Learning aim at finding the optimal policy by gradually taking local steps in the policy space, in a direction based on the the gradient of the objective function. Natural Policy Gradient (NPG) \cite{kakade2001natural} is a  policy gradient method which uses natural gradient descent \cite{raskutti2015information}, under which the gradient vector is multiplied by the inverse of the Fisher information matrix as a preconditoner \cite{agarwal2019theory}. NPG can also be thought of as a variant of mirror descent  \cite{nemirovsky1983problem, geist2019theory, shani2020adaptive},
and as a smoother approximation of PI. We present more details about these two points of view in Section \ref{sec:problem_formulation}. 

In contrast to geometric or faster convergence of VI and PI, the best known rate among policy gradient methods is $\mathcal{O}(1/k)$ convergence rate for NPG \cite{geist2019theory, agarwal2019theory}. A question of central importance in this space is the following: \textit{Do policy gradient methods achieve geometric (linear) rate of convergence to the optimal policy? }
\pagebreak
\\\textbf{Main Contributions:}
\begin{itemize}
   \item We answer the above question in affirmative, and show that vanilla NPG with constant step-size exhibits a linear rate of convergence asymptotically.  In particular, we establish finite time convergence bounds, and show that while the error rate may be  sub linear up to a threshold, it improves to a geometric rate after that. The key idea in the proof is that after the threshold, the probability of picking suboptimal actions decays geometrically fast.
   Furthermore, we consider the special case of single state MDP and show that our convergence rate is tight. 
   \item We further introduce a variant of NPG with adaptive step sizes, which has linear rate of convergence all along. Under an additional assumption on the MDP, we show that such an adaptive step size scheme can also attain super-linear convergence rate. The key idea in establishing these results is that NPG is a soft approximation of policy iteration, and the adaptive step sizes are chosen in order to ensure that the policy obtained from NPG approaches policy iteration. 
   \item Finally, we present simulation results to compare the performance of several variants of NPG algorithms presented in the paper.
\end{itemize}

\subsection{Related Work}\label{sec:related_work}

NPG algorithm was first introduced in \cite{kakade2001natural}. Several sample based variants of it for use in RL are proposed, such as Natural actor-critic \cite{peters2008natural, NIPS2009_3767, NIPS2013_5184, bhatnagar2009natural}, TRPO \cite{trpo}, PPO \cite{ppo} etc. There has been a flurry of work studying the finite-time convergence of policy gradient methods, and in particular NPG. The authors in \cite{azar2012dynamic} proposed DPP, which is a variant of PI, and they characterize its convergence bound. Later, \cite{geist2019theory} proposed MD-MPI, a variant of Mirror Descent for MDPs with constant step size $1$. They further characterized $\mathcal{O}(1/k)$ convergence of this algorithm to the optimal. The authors in \cite{agarwal2019theory} have shown that this bound can be improved by choosing a large enough step size. Furthermore, they have shown that the updates of NPG improve the policy monotonically. Later, \cite{shani2020adaptive} characterized the convergence of NPG for time-dependent step size.

While more recent work focused on establishing geometric convergence, most of the results study the NPG algorithm with some modification. In \cite{cen2020fast,lan2021policy}, linear convergence of a variant of NPG with regularization have been studied. In particular, by adding a regularizer to the objective function, the authors in \cite{cen2020fast,lan2021policy} show geometric convergence of NPG, by exploiting the strong convexity of the regularizer. By choosing an small enough regularizer coefficient, \cite{cen2020fast,lan2021policy} can ensure geometric convergence to an arbitrary close neighborhood of the global optimum. In contrast to these works, we show that the vanilla NPG has an asymptotic geometric rate of convergence even though the objective is known to be nonconvex \cite{agarwal2019theory}. The authors in \cite{bhandari2020note} show that NPG converges geometrically without any regularization, but the  step sizes have to be chosen by performing a line search. In contrast, we study NPG with constant step sizes and adaptive step sizes with explicit formulas for the step sizes.

Regarding convergence bounds for other policy gradient methods, in \cite{zhang2020variational} the authors derived a Variational Policy Gradient theorem, and show a $\mathcal{O}(1/k)$ convergence rate. In \cite{mei2020global} the authors characterized the convergence of softmax policy gradient method. They have shown $\mathcal{O}(1/k)$ and exponential convergence rate for the original and entropy regularized softmax policy gradient, respectively.

The rest of this paper is organized as follows: in Section \ref{sec:problem_formulation} we introduce the MDP formulation and NPG algorithm, in Section \ref{sec:theoretical_results} we provide the theoretical results for both constant and adaptive step sizes along with sketch of their proofs, with the complete proofs deferred to the appendix. In Section \ref{sec:exp} we present experimental results, and present concluding remarks and future work in Section \ref{sec:conclusion}. 

\section{Natural Policy Gradient for MDPs} \label{sec:problem_formulation}
We consider a Markov Decision Process (MDP) represented by a tuple $(\mathcal{S},\mathcal{A},\mathcal{R}, \mathcal{P},\gamma)$, where $\mathcal{S}$ and $\mathcal{A}$ are finite sets of states and actions, $\mathcal{R}: \mathcal{S}\times\mathcal{A}\rightarrow [0,1]$ is the random reward function, $\mathcal{P}: \mathcal{S}\times\mathcal{A}\rightarrow \Delta^{|\mathcal{S}|}$ (where $\Delta^{|\mathcal{S}|}$ is the probability simplex on $\mathbb{R}^{|\mathcal{S}|}$) is the collection of transition probabilities, and $\gamma\in(0,1)$ is the discount factor. 

The dynamics of the MDP is as follows. At each time step $t$, the system is at some state $S_t$. Given some policy $\pi:\mathcal{S}\rightarrow\Delta^{|\mathcal{A}|}$ at hand, the agent chooses an action $A_t$ by sampling from the policy as $A_t\sim\pi(\cdot|S_t)$. Given the state and the action $S_t, A_t$, the agent receives some reward $\mathcal{R}(S_t,A_t)$, and further the system transfers to the next state, which is sampled from the transition probability as $S_{t+1}\sim \mathcal{P}(\cdot|S_t,A_t)$. Since different policies incur different random trajectories over state-action pairs, the reward received by the agent depends on the policy $\pi$ which it follows. The value function $V^{\pi}(\cdot)$ is a measure of average long term reward achieved by the agent when it follows policy $\pi$, and is defined as
\[
V^\pi(s) = \mathbb{E}\left[\sum_{t=0}^\infty \gamma^t \mathcal{R}(S_t,A_t)\mid S_0\sim s, A_t\sim\pi(\cdot|S_t)\right],
\]
where the expectation is with respect to the randomness in the state-action trajectory, and the randomness in the reward function. Given some distribution over states $\rho \in \Delta^{|\mathcal{S}|}$, the average value function is defined as
\[
V^{\pi}(\rho)=\mathbb{E}_{s\sim\rho}[V^\pi(s)].
\]
The goal of the agent is to find an optimal policy $\pi^\star$, such that
\begin{align}\label{eq:pistar}
\pi^\star\in \argmax_\pi V^{\pi}(\rho).
\end{align}
Natural Policy Gradient (NPG) \cite{kakade2001natural, peters2008natural, agarwal2019theory} is an iterative algorithm to find the $\pi^\star$ in Eq. \eqref{eq:pistar}. In the rest of this section, we present the NPG update and show that it can be thought of as a smooth variant of policy improvement as well as a Mirror descent variant. Given a current policy $\pi_k$, the update of the NPG under tabular softmax parameterization is as follows
\begin{align}
    \pi_{k+1}(a|s) = \frac{\pi_k(a|s)\exp(\eta Q^{\pi_k}(s,a))}{\sum_{a'}\pi_k(a'|s)\exp(\eta Q^{\pi_k}(s,a'))}\quad \forall s,a,\label{eq:mul-weight-update2}
\end{align}
where $\eta$ is the step-size in the update and, $Q^{\pi}$ is the $Q$-function corresponding to the policy $\pi$ and is defined as
\begin{align*}
Q^\pi(s,a) = \mathbb{E}\left[\sum_{t=0}^\infty \gamma^t \mathcal{R}(S_t,A_t)|S_0= s, A_0=a, A_t\sim\pi(\cdot|S_t)\right].
\end{align*}

\subsection{Relation to Policy Iteration}\label{sec:PIPW}
Given a policy $\pi$ and a state $s$, let $\mathcal{A}^{\pi}_s$, the set of optimal actions in state $s$ with respect to the policy $\pi$, i.e., 
\begin{align}\label{eq:Optimal_action_set}
    \mathcal{A}^{\pi}_s = \argmax_{a\in \mathcal{A}} Q^\pi (s,a).
\end{align}

PI improves the current policy $\pi_k$ to a new policy $\pi_{k+1}$ by assigning zero weight to all the actions $a\notin \mathcal{A}_s^{\pi_k}$ for a state $s$, while assigning an arbitrary distribution to all $a\in\mathcal{A}_s^{\pi_k}$. However, in a sample based setting such as RL, where we do not have access to the exact $Q^{\pi_k}$ function, and it needs to be estimated from the data, implementation of PI may lead to high estimation variance, due to non smooth update of the policy. This motivates NPG as a ``smooth'' implementation of PI as follows. Instead of using a maximum with respect to $Q^{\pi_k}$ to assign zero probability to all actions $a\notin \mathcal{A}^{\pi_k}_s$, in NPG, one performs a softer variant of max as in Eq. \eqref{eq:mul-weight-update2} which assigns higher probability to more valuable actions. Note that the parameter $\eta$ tunes the softness, and as $\eta \to \infty $ in Eq. \eqref{eq:mul-weight-update2}, we recover PI. 

In other words, the higher $Q^{\pi_k}(s,a)$, the larger the update of $\pi_{k}(a|s)$ will be. Note that the denominator in Eq. \eqref{eq:mul-weight-update2} is added just to ensure that $\pi_{k+1}$ is a valid distribution function.  

\subsection{Natural Policy Gradient as Mirror Descent}\label{sec:MDPW}
Furthermore, one can derive the NPG as a variant of Mirror Descent. Given policy $\pi_k$, a step of the Mirror Descent update of the policy can be written as \cite{beck2017first}
\begin{align}
    \pi_{k+1}=\argmax_{\pi\in\Pi} \left\{\eta \langle\nabla V^{\pi_k}(\rho),\pi-\pi_k\rangle+B(\pi,\pi_k)\right\},\label{eq:MD_update}
\end{align}
where $\eta$ is the step size, and $B(\cdot,\cdot)$ is an appropriately chosen  Bregman divergence. Denote 
\[
d_\rho^{\pi}(s) = (1-\gamma)\sum_{t=0}^\infty \gamma^t\mathbb{P}(S_t=s\mid S_0\sim \rho, A_t\sim \pi(\cdot|S_t))
\]
as the discounted state visitation distribution. Replacing  $B(\pi,\pi_k)$ with $\sum_s d_\rho^{\pi_k}(s) \mathcal{KL}(\pi(\cdot|s)|\pi_k(\cdot|s))$, where $\mathcal{KL}(\cdot|\cdot)$ represents the Kullback–Leibler divergence \cite{cover1999elements}, it can be shown \cite{geist2019theory,shani2020adaptive} that the update of the policy in Eq. \eqref{eq:MD_update} is equivalent to the NPG update in Eq. \eqref{eq:mul-weight-update2}. Note that it is not known if this quantity is a well-defined Bregman divergence, and so NPG can be considered as a variant of mirror descent.

\section{Theoretical Results} \label{sec:theoretical_results}
It has been shown \cite{geist2019theory,agarwal2019theory} that the update \eqref{eq:mul-weight-update2} enjoys $V^{\pi^\star}-V^{\pi_k}\leq\mathcal{O}(1/k)$ rate of convergence.  In this section, we show that this algorithm indeed has asymptotically geometric convergence.

Before presenting our theoretical results, we introduce some useful notations. The advantage function is defined as 
\[
A^{\pi}(s,a) = Q^{\pi}(s,a)-V^\pi(s).
\]
Throughout the paper, any parameter with superscript $^\star$ corresponds to the optimal policy. For instance, $V^\star(\rho)\equiv V^{\pi^\star}(\rho)$, and $Q^\star(s,a)\equiv Q^{\pi^\star}(s,a)$. We denote the set of optimal actions in state $s$ as $\mathcal{A}_s^\star = \argmax_a Q^\star(s,a)$. 

Two useful lemmas are in order.
\begin{lemma}\label{lem:monotone_main}
For every policy $\pi$, and all $s,a$, we have
\[
 0\leq V^{\pi}(s)\leq \frac{1}{1-\gamma},\quad\quad 0\leq Q^{\pi}(s,a)\leq \frac{1}{1-\gamma}, \quad\quad \frac{-1}{1-\gamma}\leq A^{\pi}(s,a)\leq\frac{1}{1-\gamma}.
\]
\end{lemma}
\begin{lemma}\label{lem:PDL_main}
(\textup{Also known as the performance difference lemma \cite[Lemma 6.1]{kakade2002approximately}})
For any two policy $\pi_1$ and $\pi_2$, and any initial distribution $\mu$, we have
\[
V^{\pi_1}(\mu)-V^{\pi_2}(\mu) = \frac{1}{1-\gamma}\sum_{s,a} d^{\pi_1}_\mu(s)\pi_1(a|s)A^{\pi_2}(s,a).
\]
\end{lemma}

\subsection{Constant step size Natural Policy Gradient} \label{sec:con_step_size}
In this section, we show that the Natural policy gradient algorithm under constant step size converges geometrically. The pseudocode for NPG algorithm with constant step size is provided in Algorithm \ref{alg:NPG}.
\begin{algorithm}[h]\caption{NPG with constant step size}\label{alg:NPG}
\begin{algorithmic}[1] 
    \STATE {\bfseries Input:} The step size $\eta$, number of iterations $K$
	\STATE {\bfseries Initialization:} $\pi_0(a|s)=\frac{1}{|\mathcal{A}|}~\forall a,s$
	\FOR{$k=0,1,\dots,K-1$}
	\vspace{2mm}
	\STATE \quad Calculate $Q^{\pi_k}(s,a), ~~\forall s,a$
	\vspace{2mm}
	\STATE \quad$\pi_{k+1}(a|s)=\frac{\pi_k(a|s)\exp(\eta Q^{\pi_k}(s,a))}{\sum_a \pi_k(a|s)\exp(\eta Q^{\pi_k}(s,a))}~\forall a,s$
	\vspace{2mm}
	\ENDFOR
	\STATE\textbf{Output:} $\pi_K$
\end{algorithmic}
\end{algorithm}

First, we define the set of dummy states as follows
\begin{definition}
The set of of dummy states $\mathcal{S}_d$ is defined as $\mathcal{S}_d=\{s\in\mathcal{S}\mid A^\star(s,a)=0 ~\text{for all}~ a\in\mathcal{A}\}$. In other words, $\mathcal{S}_d$ is the set of states where playing any actions is optimal.
\end{definition}

Next, we define the optimal advantage function gap as follows.
\begin{definition}
\label{def:optimality_gap}
The optimal advantage function gap $\Delta$ is defined as follows:
\[
\Delta = -\max_{s\notin\mathcal{S}_d}\max_{a\notin \mathcal{A}^\star_s} A^\star(s,a),
\]
and by convention, if $\mathcal{S}_d=\mathcal{S}$, then $\Delta=\infty$.
\end{definition}
Using the classical results in MDP, we can show the following lemma:
\begin{lemma}\label{lem:delta_positive}
Let $\Delta$ be the optimal advantage function gap of an arbitrary MDP. We have $\Delta \geq 0$.
\end{lemma}
\begin{proof}
The proof follows directly from \cite[Theorem ~5.5.3]{puterman1995markov} and by noting that the optimal policy satisfies $V^{\star}(s)\geq V^{\pi}(s)$ for all $s$ and all policy $\pi$.
\end{proof}

Theorem \ref{thm:MDP} characterizes the convergence rate of NPG algorithm with constant step size.
\begin{theorem}\label{thm:MDP}
Consider Algorithm \ref{alg:NPG} with constant step size $\eta$, and $K$ number of iterations. For all $K\geq 0$, the following bound holds:
\[
V^\star(\rho)-V^{\pi_K}(\rho)\leq 
\frac{1}{(1-\gamma)^2} e^{ -(K-\kappa)(1-1/\lambda)\eta\Delta} 
\]
where, $\kappa=\frac{\lambda}{\Delta}\left[\frac{\log(|\mathcal{A}|)}{\eta}+\frac{1}{(1-\gamma)^2}\right]$, and $\lambda>1$ is an arbitrary number.
\end{theorem}

Theorem \ref{thm:MDP} shows that the NPG algorithm with constant step size converges geometrically fast asymptotically, i.e.,
\[
V^\star(\rho)-V^{\pi_K}(\rho)\leq \mathcal{O}(e^{-K(1-1/\lambda)\eta\Delta}).
\]

Note that the upper bound in Theorem \ref{thm:MDP} is useful only for large enough values of $K$. For small values of  (especially $K\leq \kappa$), the trivial bound of $V^\star(\rho)-V^{\pi_K}(\rho)\leq V^\star(\rho) \leq  1/(1-\gamma)$ from Lemma \ref{lem:monotone_main} is tighter than the bound in Theorem \ref{thm:MDP}. One should use the bounds in \cite[Theorem 5.3]{agarwal2019theory} to get tighter nontrivial bounds for small values of $K$. However, while the asymptotic rate of convergence in \cite[Theorem 5.3]{agarwal2019theory} is  $\mathcal{O}(1/K)$, our result here shows an asymptotic geometric convergence. Another quality of the bound provided in Theorem \ref{thm:MDP} is its instance dependence \cite{martin2020statistics}. In particular, the bound in Theorem \ref{thm:MDP} depends on $\Delta$, which depends to the MDP instance under study. The same parameter $\Delta$ appears in instant dependant bounds established in \cite{Marjani2020Adaptive}, where it is denoted by $\delta_{\min}$. 

Note that when $\mathcal{S}_d=\mathcal{A}$, any policy is an optimal policy. In this case,  $\Delta =\infty$, and the theorem shows that we achieve the optimal policy in a single step. 

The key idea in proving the theorem is that the probability of suboptimal actions decays exponentially after time $K=\kappa$, and we identify such a $\kappa$ using the $O(1/k)$ convergence rate established in \cite{agarwal2019theory}.  We state this result in the following lemma. 
\begin{lemma}\label{lem:pi_k_conv}
Consider the policy $\pi_K$ generated by the NPG algorithm \ref{alg:NPG}. For any $\lambda > 1$, and for $K\geq \kappa$, where $\kappa$ is defined in Theorem \ref{thm:MDP}, we have
\[
\pi_K(a|s)\leq \pi_\kappa(a|s)e^{-(K-\kappa)(1-1/\lambda)\eta\Delta} ~~\forall s\notin \mathcal{S}_d, a\notin \mathcal{A}_s^\star.
\]
\end{lemma}
Proof of Lemma \ref{lem:pi_k_conv} is provided in Appendix \ref{sec:MDP_proof}.
We now present the proof of the Theorem \ref{thm:MDP}. 

\begin{proof}[Proof of Theorem \ref{thm:MDP}]
We have
\begin{align*}
V^\star(\rho)-V^{\pi_K}(\rho) =& -(V^{\pi_K}(\rho)-V^\star(\rho))\\
=&-\frac{1}{1-\gamma}\sum_{s}d_\rho^{\pi_k}(s)\sum_a\pi_k(a|s)A^\star(s,a)\tag{By the performance difference lemma \ref{lem:PDL_main}}\\
=&-\frac{1}{1-\gamma}\left[\sum_{s\in\mathcal{S}_d}d_\rho^{\pi_k}(s)\sum_a\pi_k(a|s)A^\star(s,a)+\sum_{s\notin\mathcal{S}_d}d_\rho^{\pi_k}(s)\sum_a\pi_k(a|s)A^\star(s,a)\right]\\
=&-\frac{1}{1-\gamma}\sum_{s\notin\mathcal{S}_d}d_\rho^{\pi_k}(s)\sum_a\pi_k(a|s)A^\star(s,a)\tag{By definition of $\mathcal{S}_d$}\\
=&-\frac{1}{1-\gamma}\sum_{s\notin\mathcal{S}_d}d_\rho^{\pi_k}(s)\sum_{a\notin\mathcal{A}_s^\star}\pi_k(a|s)A^\star(s,a)\tag{By definition of $\mathcal{A}_s^\star$}\\
\leq&\frac{1}{(1-\gamma)^2}\sum_{s\notin\mathcal{S}_d}d_\rho^{\pi_k}(s)\sum_{a\notin\mathcal{A}_s^\star}\pi_k(a|s)\tag{By Lemma \ref{lem:monotone_main}}\\
\leq&\frac{1}{(1-\gamma)^2}\sum_{s\notin\mathcal{S}_d}d_\rho^{\pi_k}(s)\sum_{a\notin\mathcal{A}_s^\star}\pi_\kappa(a|s)e^{-(K-\kappa)(1-1/\lambda)\eta\Delta}\tag{For $K\geq \kappa$, and by Lemma \ref{lem:pi_k_conv}}\\
\leq&\frac{1}{(1-\gamma)^2}e^{-(K-\kappa)(1-1/\lambda)\eta\Delta}\tag{For $K\geq \kappa$}.
\end{align*}
Furthermore, for $K=\kappa$, the above upper bound is equal to $\frac{1}{(1-\gamma)^2}$. This is larger than the maximum possible value of the value function, i.e. $\frac{1}{1-\gamma}$ due to Lemma \ref{lem:monotone_main}. Hence, this bound holds for all $K\geq0$.
\end{proof}

In order to illustrate that the  the tightness of our bound, and dependency of the rate to $\Delta$, the following proposition presents a lower bound on the convergence rate for a single state MDP. 

\begin{proposition}\label{prop:lower_bound}
For an MDP with a single state, the convergence of NPG with constant step size is lower bounded as follows:
\[
V^\star-V^{\pi_K}\geq \frac{\Delta}{(1-\gamma)|\mathcal{A}|}e^{-\eta \Delta K}.
\]
\end{proposition}
Proposition \ref{prop:lower_bound} shows that the bound in Theorem \ref{thm:MDP} is tight in terms of $\Delta$. In particular, by choosing $\lambda$ large enough in Theorem \ref{thm:MDP}, we can get an arbitrary close geometric rate to the lower bound in \ref{prop:lower_bound}. The proposition is proved in Theorem \ref{thm:bandit} in Appendix \ref{sec:Bandit}, along with an alternate proof of geometric convergence of NPG in a single state MDP, by directly studying the evolution of $V^\star-V^{\pi_k}$ as a Lyapunov function. 

\subsubsection{Terminated variant of Natural Policy Gradient with constant step size}
In an MDP where $\Delta$ is known, one can run Algorithm \ref{alg:TNPG} in order to find the optimal policy.

\begin{algorithm}[h]\caption{Terminated NPG with constant step size}\label{alg:TNPG}
\begin{algorithmic}[1] 
    \STATE {\bfseries Input:} The step size $\eta$, The optimal advantage function gap $\Delta$, constant $\lambda > 1$
	\STATE {\bfseries Initialization:} $\pi_0(a|s)=\frac{1}{|\mathcal{A}|}~\forall a,s$
	\STATE $\kappa=\frac{\lambda}{\Delta}\left[\frac{\log(|\mathcal{A}|)}{\eta}+\frac{1}{(1-\gamma)^2}\right]$
	\STATE $\pi_\kappa=\text{NPG with constant step size}(\eta,\kappa)$
	\STATE $\mathcal{C}_s=\argmax_a A^{\pi_\kappa}(s,a)\quad\forall s$, if more than one maximizer, choose randomly.
	\STATE $\pi_{O}(a|s)=\begin{cases}
	1&\text{if}~~ a=\mathcal{C}_s\\
	0&\text{o.w.}
	\end{cases}$
	\STATE\textbf{Output:} $\pi_O$
\end{algorithmic}
\end{algorithm}
The following proposition gurrantees the convergence of this algorithm. 
\begin{proposition}\label{prop:terminated}
Suppose the advantage function gap $\Delta$ is known for an MDP. Consider Algorithm \ref{alg:TNPG} with $\Delta$ as an input. The output of this algorithm is the optimal policy, i.e. $\pi_O=\pi^\star$.
\end{proposition}
Assuming $\Delta$ in known, Algorithm \ref{alg:TNPG} can find the global optimal policy in a finite number of steps $\kappa=\frac{\lambda}{\Delta}\left[\frac{\log(|\mathcal{A}|)}{\eta}+\frac{1}{(1-\gamma)^2}\right]$. Picking the step size $\eta$ to be $\mathcal{O}(\log(|\mathcal{A}|)$, and $\lambda$ close to 1, we have that this algorithm needs $\mathcal{O}(\frac{1}{\Delta(1-\gamma)^2})$ iterations. To the best of our knowledge, the best upper bound on the iterations of the Howard PI to achieve a global optimal policy is $\mathcal{O}(\frac{|\mathcal{S}||\mathcal{A}|}{1-\gamma}\log(\frac{1}{1-\gamma}))$ which is characterized in \cite{scherrer2016improved}. 

\subsection{Adaptive step size Natural Policy Gradient}\label{sec:adaptive_theoretical}
In Theorem \ref{thm:MDP}, we showed that NPG algorithm converges geometrically fast asymptotically. In this section, we present a variant of the NPG algorithm which enjoys a linear rate of convergence throughout. In particular, we provide a method of increasing the step size in an adaptive manner which helps the algorithm to converge linearly from the beginning.

The idea is to approximate NPG with PI. As mentioned in \cite{bhandari2020note}, when the step size for NPG approaches to infinity, the algorithm itself converges to PI, which is known to have a linear rate of convergence. The idea can be illustrated as follows. Let $\pi$ be the current policy and consider $\mathcal{A}^{\pi}_s$ as defined in \eqref{eq:Optimal_action_set}. For the sake of the argument, assume that $|\mathcal{A}^{\pi}_s| =1$. Denote the policy obtained by one iteration of NPG with step size $\eta$ as $\pi^\eta$.
Then, as $\eta \rightarrow \infty$, 
\begin{align*}
    \pi^{\eta}(a|s) \rightarrow
    \begin{cases}
    1 &\textit{if} \quad a \in \mathcal{A}^{\pi}_s \\
    0 &\textit{otherwise.}
    \end{cases}
\end{align*}
This is equivalent to the PI update on $\pi$. Furthermore, if $|\mathcal{A}^{\pi}_s|$ is greater than 1, then as $\eta\rightarrow \infty$, $\pi^{\eta}(\cdot|s)$ converges to a distribution with support of $\mathcal{A}^{\pi}_s$, which is also equivalent to PI.

Inspired by the above observation, we develop an adaptive step size variant of NPG. This algorithm adaptively chooses a step size large enough to achieve an update close to the update of PI, and hence obtain a linear rate of convergence. Regarding the previous work, closest algorithm to our adaptive step size method is proposed in \cite{bhandari2020note}. In this work, the authors propose employment of line search in order to obtain the step-size which guarantees that the NPG update is at least as good as the PI update. A down-side of this approach is that it may be computationally infeasible to find the step size by line search. In this paper, however, we provide a simple adaptive rule to choose the step-size. Similar to optimal advantage function gap in Definition \ref{def:optimality_gap}, for any policy $\pi$ we define the gap in the Q-function as follows:
\begin{definition}
\label{def:delta_state}
For any policy $\pi$ and state $s\in \mathcal{S}$, the gap in the $Q$-function $\Delta^\pi(s)$ is defined as follows:
\[
\Delta^\pi(s) = \max_{a \in \mathcal{A}} Q^\pi(s,a) -\max_{a\notin \mathcal{A}^\pi_s} Q^\pi(s,a),
\]
where $\mathcal{A}^\pi_s$ is defined in \eqref{eq:Optimal_action_set}. In the case that $\mathcal{A}^{\pi}_s=\mathcal{A}$, we take $\max_{a\notin \mathcal{A}^\pi_s} Q^\pi(s,a)=-\infty$. 
\end{definition}

Note that by the definition of $\Delta^{\pi}$, we have $\Delta^{\pi}(s)>0$ for any $s \in \mathcal{S}$. Using the definition of $\Delta^{\pi_k}$, at time step $k$ we take the step size $\eta_k$ as follows:
\begin{equation}
    \eta_k \geq \max_{s \in S, a \in \mathcal{A}^{\pi_k}_s}\left\{\left(L_k+\log\left(\frac{|\mathcal{A}|}{\pi_k(a|s)}\right)\right)\frac{1}{\Delta^{\pi_k}(s)}\right\}, \label{eq: adaptive_step_size}
\end{equation}
where $L_k>0$ is a function of $k$. Later in this section we will consider two cases $L_k=Lk$ and $L_k=L$ separately. Using the step size $\eta_k$ we run NPG in an adaptive manner, as presented in Algorithm \ref{alg:NPG_adaptive}.

\begin{algorithm}[h]\caption{NPG with adaptive step size}\label{alg:NPG_adaptive}
\begin{algorithmic}[1] 
	\STATE {\bfseries Initialization:} $\pi_0(a|s)=\frac{1}{|\mathcal{A}|}~\forall a,s$
	\FOR{$k=0,1,\dots,K-1$}
	\vspace{2mm}
	\STATE \quad Calculate $Q^{\pi_k}(s,a)~\forall a,s$
	\STATE \quad Pick $\eta_k$ such that
	\begin{equation*}
	    \eta_k \geq \left\{\left(L_k+\log\left(\frac{|\mathcal{A}|}{\pi_k(a|s)}\right)\right)\frac{1}{\Delta^{\pi_k}(s)}\right\}
	\end{equation*}
	\quad for all $ s \in S$ and $ a \in \mathcal{A}^{\pi_k}_s$
	\vspace{2mm}
	\STATE \quad$\pi_{k+1}(a|s)=\frac{\pi_k(a|s)\exp(\eta_k Q^{\pi_k}(s,a))}{\sum_a \pi_k(a|s)\exp(\eta_k Q^{\pi_k}(s,a))}~\forall a,s$
	\vspace{2mm}
	\ENDFOR
	\STATE\textbf{Output:} $\pi_K$
\end{algorithmic}
\end{algorithm}

\subsubsection{Geometric convergence with Adaptive step-size}
Now, we present the theorem which shows that using the  adaptive step-size, NPG converges linearly.
\begin{theorem}
\label{thm: npg_adaptive}
Consider the NPG algorithm as given in Algorithm \ref{alg:NPG_adaptive} with $k$ iterations and with adaptive step size $\eta_k$ as given by Eq. \eqref{eq: adaptive_step_size}, then

\begin{enumerate}
    \item For $L_k=Lk$ for some $L>0$, we have
    \begin{align}\label{eq:Adaptive_bound}
        V^{\star}(\rho)-V^{\pi_k}(\rho) \leq \gamma^k & (V^{\star}(\rho)-V^{\pi_{0}}(\rho))+\frac{1}{1-\gamma} \beta(k)\max\left\{\gamma^{k},\exp\{-Lk\}\right\},
    \end{align}
    where 
    \begin{equation*}
    \beta(k)=\begin{cases}
    \frac{1}{|\gamma-\exp\{-L\}|} &\textit{if } L \neq -\log(\gamma) \\
    k\exp\{L\} &\textit{otherwise}.
    \end{cases}
    \end{equation*}
    
    \item For $L_k =L$ for all $k\geq 0$, then,
    \begin{align}
    \label{eq: adaptive_noninc_bound}
        |V^{\pi^\star}(\rho)-V^{\pi_k}(\rho)| \leq \gamma^k & |V^{\pi^\star}(\rho)-V^{\pi_{0}}(\rho)|+\frac{1}{(1-\gamma)^2}\exp\big(-L\big).
    \end{align}
\end{enumerate}

\end{theorem}
Intuitively, the first term in the right hand side of Eq. \eqref{eq:Adaptive_bound} or \eqref{eq: adaptive_noninc_bound} appears due to the gap from the optimum after $k$ updates of PI and the second term is the cumulative error introduced in $k$ steps due to the gap between NPG update and PI update. Note that in Eq. \eqref{eq:Adaptive_bound}, $\gamma^k \geq \exp\{-Lk\}$ if we choose $L \geq -\log \gamma$ and in this case, the rate of convergence is very similar to that of PI. Due to this, later in Section \ref{sec:exp}, we choose $L = -\log \gamma$ for simulations. Eq. \eqref{eq: adaptive_noninc_bound} shows that if we run NPG with an adaptive step size as in \eqref{eq: adaptive_step_size} with constant parameter $L$, NPG will converge linearly to a ball around the optimum  with radius $\mathcal{O}(\exp(-L))$. By tuning the parameter $L$ in the step-size we can control the final error.

We would like to point out that we achieve similar bounds as in \cite{bhandari2020note} without using line search to choose the step-size for the NPG update. This is an attractive property in practice as line search may be computationally inefficient. Thus, each iteration of the algorithm presented in \cite{bhandari2020note} will take more time to implement.

To prove Theorem \ref{thm: npg_adaptive}, first we will bound the gap between PI update and NPG update. In particular, let $\Tilde{\pi}_{k+1}$ be the policy obtained after one iteration of PI on $\pi_k$. By definition, the PI step will choose a policy with support $\mathcal{A}_s^\pi$. We pick $\Tilde{\pi}_{k+1}$ as the following specific distribution over $\mathcal{A}_s^\pi$:
\begin{align}
    \Tilde{\pi}_{k+1}(a|s)=
    \begin{cases}
    \frac{\pi_k(a|s)} {\displaystyle\sum_{a \in \mathcal{A}^{\pi_k}_s} \pi_k(a|s)} &\textit{if} \quad a \in \mathcal{A}^{\pi_k}_s \\
    0 &\textit{otherwise.}
    \end{cases} \label{eq: PI_update}
\end{align}
The reason to pick the specific update of the PI in \eqref{eq: PI_update} is that the NPG update converges to the policy given by Eq. \eqref{eq: PI_update} as the step-size increases to infinity. Now, we present the difference between value function corresponding to the NPG update and the PI update in the following lemma:
 \begin{lemma}
\label{lemma: npg_pi}
Consider the NPG as given in Algorithm \ref{alg:NPG_adaptive} with adaptive step size $\eta_k$ as given by Eq. \eqref{eq: adaptive_step_size}. Then,
\begin{equation*}
    \left| V^{\pi_{k+1}}(\rho) - V^{\Tilde{\pi}_{k+1}}(\rho) \right | \leq \frac{1}{1-\gamma} \exp\{-L_{k}\}.
\end{equation*}
\end{lemma}
The proof of Lemma \ref{lemma: npg_pi} is provided in Appendix \ref{sec:Thm_3_proof}. Next, we use the above lemma to prove Theorem \ref{thm: npg_adaptive}.
 \begin{proof}[Proof of Theorem \ref{thm: npg_adaptive}]

By using the triangle inequality, we have
\begin{align}
\label{eq: v_recurs}
    \left| V^{\star} (\rho) - V^{\pi_{k}}(\rho) \right| &\leq | V^{\star} (\rho)  - V^{\Tilde{\pi}_{k}}(\rho) |+  | V^{\pi_{k}} (\rho) - V^{\Tilde{\pi}_{k}}(\rho) |\nonumber\\
    & \stackrel{(a)}{\leq} \gamma | V^{\star} (\rho)  - V^{\pi_{k-1}}(\rho) | + | V^{\pi_{k}} (\rho) - V^{\Tilde{\pi}_{k}}(\rho) |\nonumber\\
    & \stackrel{(b)}{\leq} \gamma | V^{\star} (\rho)  - V^{\pi_{k-1}}(\rho) | +\frac{1}{1-\gamma} \exp(-L_{k-1})\nonumber\\
    &\stackrel{(c)}{\leq} \gamma^{k} | V^{\star} (\rho)  - V^{\pi_{0}}(\rho) |
      + \frac{1}{1-\gamma} \sum_{i=0}^{k-1} \gamma^i \exp\big(-L_{k-i-1}\big),
\end{align}
where (a) follows by the contraction property of PI (e.g. see: \cite[Section 3]{bhandari2020note} and references therein), (b) follows because of Lemma \ref{lemma: npg_pi} and (c) follows by solving the recursion in (b). 

Now if we choose $L_k = Lk$, it can be shown that
\begin{align*}
   \sum_{i=0}^{k-1} \gamma^i \exp\big(-L_{k-i-1}\big)=& \sum_{i=0}^{k-1} \gamma^i \exp\big(-L(k-i-1)\big) \\
   \stackrel{(a)}{\leq}& \beta(k)\max\left\{\gamma^{k},\exp\{-Lk\}\right\}.
\end{align*}
The calculation of the inequality $(a)$ is given in the Appendix \ref{sec:Thm_3_proof}.

Also, for $L_{k} = L$ for all $i$,
\begin{align*}
   \sum_{i=0}^{k-1} \gamma^i \exp\big(-L_{k-i-1}\big)=& \exp(-L) \sum_{i=0}^{k-1} \gamma^i \leq \frac{1}{1-\gamma} \exp(-L)
\end{align*}
This completes the proof of Theorem \ref{thm: npg_adaptive}.
\end{proof}

\subsection{Super-linear Convergence of NPG}
In this section, we show that under certain conditions, NPG enjoys a quadratic rate of convergence, asymptotically. Before presenting the result, we introduce some notations following the exposition in \cite{puterman1979convergence}. For any given policy $\pi$, let $T_\pi=\gamma P_\pi-I$ where $I$ is the identity matrix and $P_\pi$ is the transition probability matrix corresponding to the policy $\pi$, i.e., $[P_\pi]_{s,s'}=\sum_{a \in \mathcal{A}}\mathcal{P}(s'|s,a)\pi(a|s)$. Also, let $\|\cdot\|_\rho$ be the weighted $L_1$-norm with weights given by $\rho$. In particular, for a vector $V\in\mathbb{R}^{|\mathcal{S}|}$, we have $\|V\|_\rho=\sum_{s \in S} \rho(s)|V(s)|$. In addition, for a matrix $T$, define the induced weighted $L_1$-norm as $\|T\|_\rho=\sup_{\|v\|_\rho \leq 1} \|Tv\|_\rho$. For any policy $\pi$, $V^\pi$ denotes the vector of length $|\mathcal{S|}$ with $V^\pi(s)$ as its elements. We have the following result.
\begin{theorem} \label{theorem: asymptotic_adaptive} Let $\{\pi^k\}_{k \in \mathbb{Z}_+}$ be the sequence of policies achieved by the iterates of NGP, and assume that there exists $0<\tilde{L},M<\infty$ and $p \in (1,2]$ such that 
\begin{align}
    \|T_{\pi^k}-T_{\pi^\star}\|_\rho \leq \tilde{L}\|V^{\pi_k}-V^\star\|^{p-1}_\rho, \quad \|T^{-1}_{\pi^k}\|_\rho \leq M \label{eq: lipshitz}
\end{align}
for $k\in \mathbb{Z}_+$. Let $b = \tilde{L}M$  and pick $L_k = Lp^k$, then
\begin{align*}
    |V^{\star}(\rho)-V^{\pi_k}(\rho)| \leq \exp \left( p^{k} \left( \frac{\log b}{p-1} + \frac{p \log p}{(p-1)^2}\right) \right)\left( |V^\star(\rho)-V^{\pi_{0}}(\rho)|^{p^{k}} +   k \exp \left( p^k \left( -\log (1-\gamma) -L \right) \right) \right).
\end{align*}
Thus, under the condition that 
\begin{align}
    \log \left(|V^\star(\rho)-V^{\pi_{0}}(\rho)|\right) < - \left( \frac{\log b}{p-1} + \frac{p \log p}{(p-1)^2}\right), \quad L > \frac{\log b}{p-1} + \frac{p \log p}{(p-1)^2} -\log (1-\gamma), \label{eq: close_enough}
\end{align}
NPG converges super-linearly with order $p$.
\end{theorem}
Note that, by considering $p=2$, NPG enjoys a quadratic rate of convergence. The quadratic convergence is manifested by the following two key steps: (1) quadratic convergence of NPG to policy iteration and (2) quadratic convergence of policy iteration to the optimum. 
The former is achieved by considering an exponentially increasing step-size $L_k = Lp^k$ in Eq. \eqref{eq: adaptive_step_size}, which along with the update equation of NPG, results in a quadratic convergence of NPG to policy iteration. Same result is valid when the step size is picked as $L_k = L \alpha^k$ where $\alpha> p$.

To motivate the necessity of condition \eqref{eq: lipshitz}, note that \cite[Theorem 2]{puterman1979convergence} shows that policy iteration is equivalent to Newton's method for operator equations, and \cite[Theorem 4]{puterman1979convergence} shows the quadratic convergence of policy iteration. In addition, we know that Newton's method enjoy a quadratic rate of convergence, assuming Lipchitz continuity of the gradient of the objective function. The Lipchitz continuity assumption in the Newton's method is equivalent to \eqref{eq: lipshitz} with $p=2$.

Condition \eqref{eq: close_enough} ensures that the initial value function $V^{\pi_0}$ is close to the optimal value function $V^\star$ and the step-size is large enough. In particular, Theorem \ref{theorem: asymptotic_adaptive} is an asymptotic result which shows that NPG enjoys quadratic rate of convergence close to the optimum. Thus, combining Theorem \ref{thm: npg_adaptive} and Theorem \ref{theorem: asymptotic_adaptive}, we conclude that NPG converges linearly given any arbitrary starting policy and enjoys a quadratic rate of convergence once the value function of the current iterate is close to the optimum. Next, we present the proof of the Theorem \ref{theorem: asymptotic_adaptive} below:
\begin{proof}[Proof of Theorem \ref{theorem: asymptotic_adaptive}]
By using the triangle inequality, we have
\begin{align*}
    |V^\star(\rho)-V^{\pi_{k+1}}(\rho)| &\leq |V^\star(\rho)-V^{\tilde{\pi}_k}(\rho)|+|V^{\pi_k}(\rho)-V^{\tilde{\pi}_k}(\rho)| \\
    &\overset{(a)}{\leq} |V^\star(\rho)-V^{\tilde{\pi}_k}(\rho)|+\frac{1}{1-\gamma}\exp(-L_{k}) \\
    &\overset{(b)}{=} \|V^\star-V^{\tilde{\pi}_k}\|_\rho+\frac{1}{1-\gamma}\exp(-L_{k})  \\
    &\overset{(c)}{\leq} b\|V^\star-V^{\pi_{k-1}}\|^p_\rho+\frac{1}{1-\gamma}\exp(-L_{k})\\
    &= b|V^\star(\rho)-V^{\pi_{k-1}}(\rho)|^p+\frac{1}{1-\gamma}\exp(-L_{k}) \\
    & \overset{(d)}{\leq} \exp \left( p^{k+1} \left( \frac{\log b}{p-1} + \frac{p \log p}{(p-1)^2}\right) \right)|V^\star(\rho)-V^{\pi_{0}}(\rho)|^{p^{k+1}} \\
    & \ \ +   k \exp \left( p^k \left( \frac{\log b}{p-1} + \frac{p \log p}{(p-1)^2} - \log (1-\gamma) -L \right) \right),
\end{align*}
where $(a)$ follows by Lemma \ref{lemma: npg_pi}, $(b)$ follows by noting that $V^\star \geq V^{\tilde{\pi}_k}$ component wise and the definition of the weighted $L_1$-norm $\|\cdot\|_\rho$. Next, $(c)$ follows by the order $p$ convergence of PI proved in \cite[Theorem 4]{puterman1979convergence}. Finally, $(d)$ follows by repeating the same procedure $k$ times, the computations for which are presented in Appendix \ref{app: quadratic_convergence}. Now, if $L$ is large enough and $|V^\star(\rho)-V^{\pi_0}|$ is small enough (as given in Eq. \eqref{eq: close_enough}), the error will decay to zero as $k$ increases, and the proof is complete.
\end{proof}

\section{Simulation}\label{sec:exp}
In this section, we compare the algorithms presented in this paper with the state of the art and plot the error versus the number of iterations in Fig. \ref{fig:my_label}. We construct an MDP with 70 states and 10 actions. We generate the reward for each state-action pair uniformly at random between 0 and 1. Similarly, we construct the transition probability matrix by generating a uniform random number between 0 and 1 for each entry and then re-scaling each row to obtain a stochastic matrix.

We simulate 5 different algorithms as listed below:
\begin{itemize}
    \item PG: Policy Gradient \cite{agarwal2019theory} with constant step size where step size was choosen to be $\frac{(1-\gamma)^3}{2|\mathcal{A}|\gamma}$.
    \item NPG-C: Natural Policy Gradient with constant step size as given in Algorithm \ref{alg:NPG}. Here we choose the step size $\eta = \log |\mathcal{A}|$.
    \item NPG-I: Natural Policy Gradient with increasing step size where we pick $\eta_k = -k\log \gamma $. Note that here we are not using the adaptive part while updating the step size.
    \item NPG-A: Natural Policy Gradient with adaptive step size as given in Algorithm \ref{alg:NPG_adaptive} with $L_k = -\log \gamma$. 
    \item NPG-AI: Natural Policy Gradient with adaptive step size as given in Algorithm \ref{alg:NPG_adaptive} with $L_k = -k\log \gamma$.
\end{itemize}
We run these algorithms over the same MDP with the same initial policy. The $x$-axis in Fig. \ref{fig:my_label} denotes the number of iterations and $y$-axis denotes the error in the value function of the current policy with respect to the optimal policy, i.e.,
\begin{equation*}
    \text{Error} = |\mathcal{S}| (V^*(\rho) - V^{\pi_k}(\rho)),
\end{equation*}
where we pick $\rho(s) = 1/|\mathcal{S}|$ for all $s$.
\begin{figure}[h]
    \centering
    \includegraphics[scale=0.90]{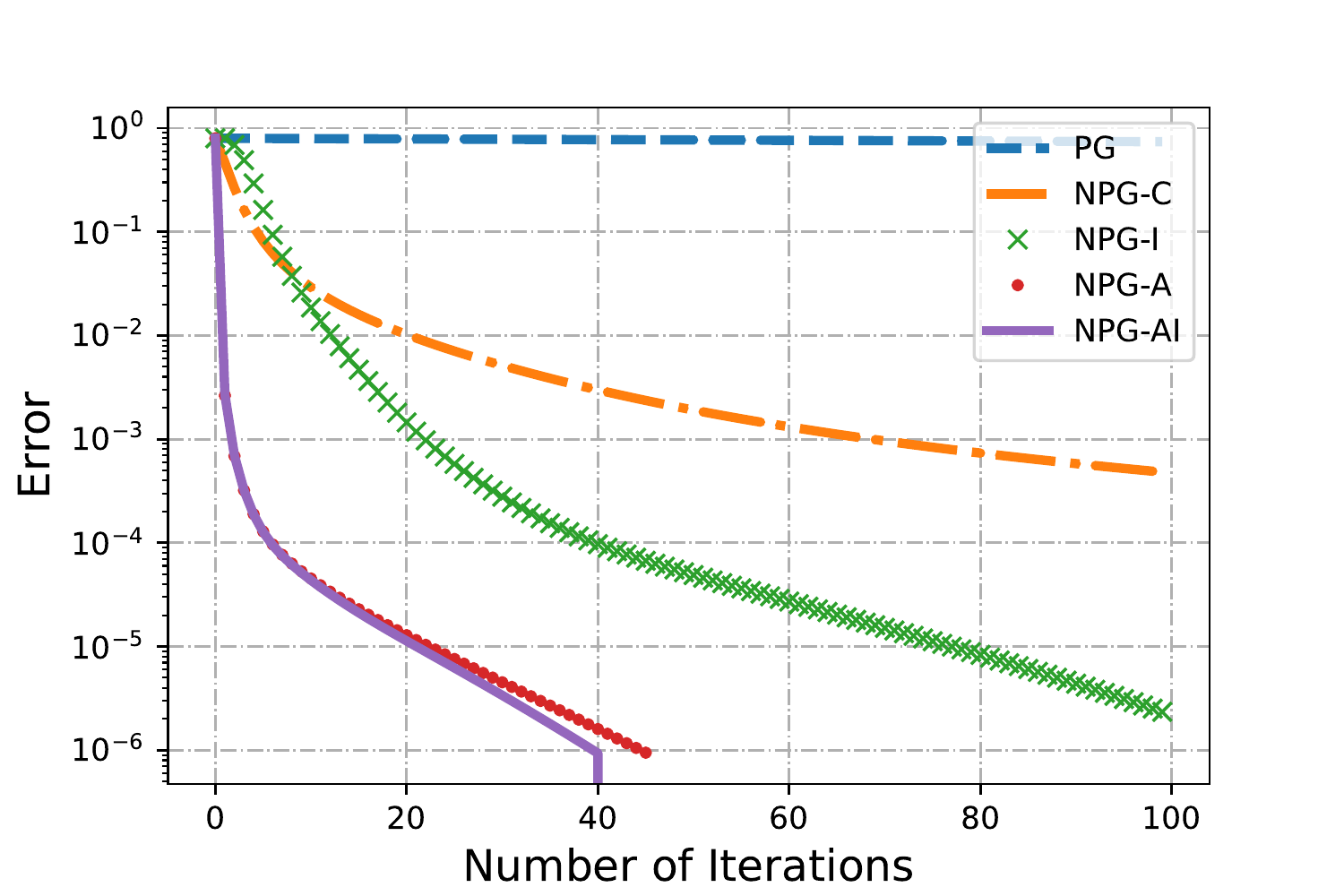}
    \caption{\textit{Error} vs \textit{Number of iterations} plot for the 5 different algorithms on an MDP with 70 states and 10 actions. Here, $x$-axis is on linear scale and $y$-axis is on logarithmic scale.}
    \label{fig:my_label}
\end{figure}

It can be observed from Fig. \ref{fig:my_label} that NPG-A and NPG-I both perform better than NPG-C. This shows that using adaptive step size or using the increasing step size, both improves the performance of Natural Policy Gradient. Furthermore, if we use both adaptive and increasing step size, as in NPG-AI, we get even better performance. Moreover, NPG-A and NPG-AI, both show significant improvement in just one iteration.

\section{Conclusion and Future work}\label{sec:conclusion}
In this paper we studied the Natural Policy Gradient Algorithm, and obtained its finite sample convergence error bounds. We have shown that NPG with constant step size converges geometrically fast asymptotically. In particular, we showed that the rate of $O(1/k)$ upto a threshold, after which it improves to geometric rate. Furthermore, we proposed a variant of NPG with adaptive step size, and showed that it converges geometrically from the beginning. Finally, we have compared these algorithms numerically.

An immediate future work is to employ our analysis for a sample based setting implementation of NPG, also know as Natural actor-critic. Recently, there has been a line of work on the analysis of actor-critic type algorithms \cite{qiu2019finite, shani2020adaptive, wu2020finite, xu2020improving, khodadadian2021finite, khodadadian2021finite2}, where \cite{khodadadian2021finite2} characterizes the best convergence result of $\mathcal{O}(1/k^{1/3})$. By employing the improved convergence rate of NPG proposed in this paper, we believe that it is possible to improve the rate of the stochastic variant.

\bibliographystyle{abbrv}
\bibliography{arXiv.bib}
\pagebreak
\appendix
\appendixpage
\section{Useful Lemmas}
\begin{lemma}\label{lem:monotone}
The value function and the $Q$-function corresponding to the policies generated by the NPG Algorithm \ref{alg:NPG} with constant step size is monotonically increasing, i.e.,
\begin{align*}
V^{\pi_k}(s)&\leq V^{\pi_{k+1}}(s)~\forall s,\\
Q^{\pi_k}(s,a)&\leq Q^{\pi_{k+1}}(s,a)~\forall s,a.    
\end{align*}

Furthermore,
\begin{align}\label{eq:Q_opt_bound}
Q^{\pi_k}(s,a) \leq Q^*(s,a) ~\forall s,a,k,
\end{align}
\end{lemma}

\begin{proof}[Proof of Lemma \ref{lem:monotone}]
The monotonicity of the value function follows directly from \cite[Lemma 5.2]{agarwal2019theory}. Furthermore, for the $Q$-function we have
\begin{align*}
    Q^{\pi_k}(s,a) =& \E[\mathcal{R}(s,a)] + \gamma \sum_{s'}P(s'|s,a)V^{\pi_k}(s')\\
    \leq&\E[\mathcal{R}(s,a)] + \gamma \sum_{s'}P(s'|s,a)V^{\pi_{k+1}}(s')\\
    =&Q^{\pi_{k+1}}(s,a),
\end{align*}
where the inequality follows from the monotonicity of the value function. The inequality in \eqref{eq:Q_opt_bound} is a direct implication of monotonic increase of the $Q^{\pi_k}(s,a)$ for all $s,a$.

\end{proof}

\begin{proof}[Proof of Lemma \ref{lem:monotone_main}]

By definition of the value function, we have

\begin{align*}
V^\pi(s) =& \mathbb{E}\left[\sum_{t=0}^\infty \gamma^t \mathcal{R}(S_t,A_t)\mid S_0\sim s, A_t\sim\pi(\cdot|S_t)\right]\\
\leq & \mathbb{E}\left[\sum_{t=0}^\infty \gamma^t \mid S_0\sim s, A_t\sim\pi(\cdot|S_t)\right]\tag{$\mathcal{R}(S_t,A_t)\leq 1$ by construction}\\
=&\sum_{t=0}^\infty \gamma^t\\
=&\frac{1}{1-\gamma}.
\end{align*}
A similar argument holds for $Q^\pi(s,a)$. Furthermore, $V^{\pi}(s)\geq0$ and $Q^{\pi}(s,a)\geq 0$ follows directly from $\mathcal{R}(S_t,A_t)\geq 0$. In addition, from the definition of $A^\pi(s,a)$, and the non negativity of $V^\pi(s)$, we have 
\[
A^{\pi}(s,a) = Q^{\pi}(s,a)-V^\pi(s)\leq Q^{\pi}(s,a) \leq \frac{1}{1-\gamma}.
\]
$A^{\pi}(s,a)\geq\frac{-1}{1-\gamma}$ follows similarly.
\end{proof}

\begin{lemma}\label{lem:alekh_agarwal}
\textup{(Global Convergence for NPG \cite[Theorem 5.3]{agarwal2019theory})}
Consider the policy $\pi_k$ achieved in the k'th iteration of Algorithm \ref{alg:NPG}. For any given initial distribution $\rho$, we have the following bound for all $k\geq 0$:
\[
V^*(\rho)-V^{\pi_k}(\rho) \leq \left(\frac{\log(|\mathcal{A}|)}{\eta}+\frac{1}{(1-\gamma)^2}\right)\frac{1}{k}.
\]
\end{lemma}
\section{Details of the proof of Theorem \ref{thm:MDP}} \label{sec:MDP_proof}

\begin{proof}[Proof of Lemma \ref{lem:pi_k_conv}]
By definition of advantage function, we have 
\begin{align}
A^{\pi_k}(s,a)-A^*(s,a) = & V^*(s)-V^{\pi_k}(s) -(Q^*(s,a)-Q^{\pi_k}(s,a))\nonumber\\
\leq& V^*(s)-V^{\pi_k}(s)\tag{Lemma \ref{lem:monotone}}\nonumber\\
\leq & \left(\frac{\log(|\mathcal{A}|)}{\eta}+\frac{1}{(1-\gamma)^2}\right)\frac{1}{k}\tag{Lemma \ref{lem:alekh_agarwal}}\nonumber\\
=& \frac{C}{k},\label{eq:Abound}
\end{align}
where $C=\frac{\log(|\mathcal{A}|)}{\eta}+\frac{1}{(1-\gamma)^2}$. 

For $k\geq \kappa=\frac{\lambda C}{\Delta}$, we have
\begin{align}
\frac{C}{k}\leq & \frac{\Delta}{\lambda} \nonumber\\
=& \frac{-1}{\lambda}\max_{s\notin\mathcal{S}_d,a\notin \mathcal{A}^*_s} A^{*}(s,a) \nonumber\\
=& \frac{1}{\lambda}\min_{s\notin\mathcal{S}_d,a\notin \mathcal{A}^*_s} -A^*(s,a)\nonumber\\
\implies \frac{C}{k}\leq & \frac{-A^*(s,a)}{\lambda}~~\forall s\notin\mathcal{S}_d,a\notin \mathcal{A}^*_s, k\geq\kappa.\label{eq:C/tbound}
\end{align}
Combining \eqref{eq:Abound} and \eqref{eq:C/tbound}, we get:
\begin{align}
A^{\pi_k}(s,a) \leq &(1-1/\lambda)A^{*}(s,a) \nonumber\\
\leq &-(1-1/\lambda)\Delta ~~\forall s\notin \mathcal{S}_d,a\notin \mathcal{A}^*_s, k\geq \kappa.\label{eq:Apitbound}
\end{align}
Furthermore, by applying $\ln(\cdot)$ on both sides of the update of $\pi_k$
\[
\pi_{k+1}(a|s) = \frac{\pi_k(a|s)\exp(\eta Q^{\pi_k}(s,a))}{\sum_{a'}\pi_k(a'|s)\exp(\eta Q^{\pi_k}(s,a'))},
\]
we get:
\begin{align*}
\ln(\pi_{k+1}(a|s)) = &\ln(\pi_{k}(a|s)) + \eta Q^{\pi_k}(s,a) - \ln\left(\sum_{a'}\pi_k(a'|s)\exp(\eta Q^{\pi_k}(s,a'))\right)\\
 \leq &\ln(\pi_{k}(a|s)) + \eta Q^{\pi_k}(s,a) - \eta V^{\pi_k}(s)\\
 =&\ln(\pi_{k}(a|s)) + \eta A^{\pi_k}(s,a),
\end{align*}
where the inequality is due to concavity of $\ln(\cdot)$ function and Jensen's inequality.

Recursively applying the previous inequality, and applying $\exp(\cdot)$ on both sides, for $K\geq \kappa$, we get:
\begin{align}
\pi_K(a|s) \leq & \pi_{\kappa}(a|s) \exp\left(\eta\sum_{k=\kappa}^{K-1} A^{\pi_i}(s,a)\right)\nonumber\\
\leq & \pi_{\kappa}(a|s) \exp\left(-(K-\kappa)(1-1/\lambda)\eta\Delta\right)\quad\forall s\notin \mathcal{S}_d,a\notin \mathcal{A}^*_s, K\geq \kappa\tag{By inequality \eqref{eq:Apitbound}}\nonumber
\end{align}

\end{proof}

\begin{proof}[Proof of Proposition \ref{prop:terminated}]
As shown in the proof of Lemma \ref{lem:pi_k_conv}, for all $k\geq \kappa$, we have
\begin{align}
A^{\pi_k}(s,a) \leq -(1-1/\lambda)\Delta ~~\forall s\notin \mathcal{S}_d,a\notin \mathcal{A}^*_s, k\geq \kappa.\label{eq:Apitbound2}
\end{align}
By definition of the advantage function, for all $s$, we have $\sum_a A^{\pi_\kappa}(s,a)\pi_\kappa(a|s)=0$. Combining this with the result in \ref{eq:Apitbound2}, we can show that $\argmax_aA^{\pi_\kappa}(s,a)=\argmax_{a\in\mathcal{A}^*_s}A^{\pi_\kappa}(s,a)\subset \mathcal{A}^*_s$ for all $s$, and hence $\argmax_aA^{\pi_\kappa}(s,a)\subset \mathcal{A}^*_s$, which completes the proof.
\end{proof}

\section{Single state MDP}\label{sec:Bandit}
In this section we provide the convergence of NPG with constant step size for an MDP with a single state. Due to single state MDP, we can work directly with the objective function itself, instead of the performance difference lemma \ref{lem:PDL_main}, and we can prove the convergence bound using a Lyapunov type argument. Furthermore, we prove the lower bound in Proposition \ref{prop:lower_bound}.

Denote $r^* = \max_a r(a)$, and $\mathcal{A}^*=\{a|r(a)=r^*\}$. Since we assume a single state, throughout this section we do not denote the state space. We have $Q^\star(a)=r(a)+\gamma V^\star$, and $V^\star=\frac{r^\star}{1-\gamma}$. Hence, $A^\star(a)=Q^\star(a)-V^\star=r(a)+\gamma V^\star-V^\star=-(r^\star-r(a))$. Hence, the optimality gap $\Delta$ is as follows:
\[
\Delta = r^*-\max_{a\notin \mathcal{A}^*} r(a).
\]

We have the following theorem:

\begin{theorem}\label{thm:bandit}
Suppose we run Algorithm \ref{alg:NPG} with constant step size for an MDP with a single state. For any $\lambda>0$, denote $\kappa_b = \left\lceil \frac{C(1-\gamma)/\Delta}{1-e^{-\eta \Delta\lambda}}\right\rceil$. For all $K\geq \kappa_b$, with $C=\left(\frac{1}{(1-\gamma)^2}+\frac{\log|\mathcal{A}|}{\eta}\right)$, we have
\[
\frac{1}{1-\gamma}\frac{\Delta}{|\mathcal{A}|}e^{-\eta K \Delta}\stackrel{\text{(a)}}{\leq} V^*-V^{\pi_{K}}\leq e^{-\eta \Delta (1-\lambda)(K-\kappa_b)}(V^*-V^{\pi_{\kappa_b}}) \leq \frac{1}{1-\gamma} e^{-\eta \Delta (1-\lambda)(K-\kappa_b)}.
\]
Furthermore, the lower bound in $(a)$ holds for all $K\geq 0$.

\end{theorem}

\begin{proof}[Proof of Theorem \ref{thm:bandit}]
By definition, $V^*-V^{\pi_{k}} = \frac{1}{1-\gamma}\sum_a\pi_k(a)(r^*-r(a))$, and the update of the policy is follows
\[
\pi_{k+1}(a) = \frac{\pi_k(a)\exp(\eta r(a))}{\sum_{a'}\pi_k(a')\exp(\eta r(a'))}.
\]
Using $V^*-V^{\pi_{k}}$ as the Lyapunov function, we have
\begin{align*}
    V^*-V^{\pi_{k+1}} = &\frac{1}{1-\gamma}\frac{\sum_a \pi_k(a)\exp(\eta(r(a)-r^*))(r^*-r(a))}{\sum_a \pi_k(a)\exp(\eta(r(a)-r^*))}\\
    = &\frac{1}{1-\gamma}\left[\frac{\sum_{a\in \mathcal{A}^*} \pi_k(a)\exp(\eta(r(a)-r^*))(r^*-r(a))}{\sum_a \pi_k(a)\exp(\eta(r(a)-r^*))} + \frac{\sum_{a\notin \mathcal{A}^*} \pi_k(a)\exp(\eta(r(a)-r^*))(r^*-r(a))}{\sum_a \pi_k(a)\exp(\eta(r(a)-r^*))}\right]\\
    = &\frac{1}{1-\gamma}\frac{\sum_{a\notin \mathcal{A}^*} \pi_k(a)\exp(\eta(r(a)-r^*))(r^*-r(a))}{\sum_a \pi_k(a)\exp(\eta(r(a)-r^*))}\\
    \leq &\frac{1}{1-\gamma}\frac{\exp(-\eta\Delta)\sum_{a\notin \mathcal{A}^*} \pi_k(a)(r^*-r(a))}{\sum_a \pi_k(a)\exp(\eta(r(a)-r^*))}\\
    = &\frac{1}{1-\gamma}\frac{\exp(-\eta\Delta)\sum_{a} \pi_k(a)(r^*-r(a))}{\sum_a \pi_k(a)\exp(\eta(r(a)-r^*))}\\
    = &\frac{V^*-V^{\pi_{k}}}{\sum_a \pi_k(a)\exp(\eta(r(a)-r^*+\Delta))}.
\end{align*}
Using Lemma \ref{lem:alekh_agarwal}, we have
\begin{align*}
\frac{C}{k}\geq &V^*-V^{\pi_{k}}\\
=&\frac{1}{1-\gamma}\sum_{a\notin \mathcal{A}^*} \pi_k(a)(r^*-r(a))\\
\geq&\frac{\Delta}{1-\gamma}\sum_{a\notin \mathcal{A}^*} \pi_k(a),
\end{align*}
where $C= \left(\frac{1}{(1-\gamma)^2}+\frac{\log|\mathcal{A}|}{\eta}\right)$. Hence we have 
\[
\sum_{a\in \mathcal{A}^*} \pi_k(a) \geq 1-\frac{C(1-\gamma)}{k\Delta}.
\]
Using the above bound we have
\begin{align*}
    \sum_a \pi_k(a)\exp(\eta(r(a)-r^*+\Delta)) = &\sum_{a\in \mathcal{A}*} \pi_k(a)\exp(\eta\Delta) + \sum_{a\notin \mathcal{A}^*} \pi_k(a)\exp(\eta(r(a)-r^*+\Delta))\\
    \geq & \exp(\eta\Delta) \left(1-\frac{C(1-\gamma)}{k\Delta}\right)
\end{align*}
For $k\geq \kappa_b = \left\lceil \frac{C(1-\gamma)/\Delta}{1-e^{-\eta \Delta\lambda}}\right\rceil$, we have
\begin{align*}
    \sum_a \pi_k(a)\exp(\eta(r(a)-r^*+\Delta))  \geq & \exp(\eta\Delta(1-\lambda)).
\end{align*}
Hence, we have
\[
V^*-V^{\pi_{k+1}}\leq e^{-\eta \Delta (1-\lambda)}(V^*-V^{\pi_{k}}).
\]
As a result,
\[
V^*-V^{\pi_{K}}\leq e^{-\eta \Delta (1-\lambda)(K-\kappa_b)}(V^*-V^{\pi_{\kappa_b}}), ~~~ \forall K\geq \kappa_b.
\]

Next we provide the proof of Proposition \ref{prop:lower_bound}, which constructs the lower bound. 

Due to the update of the policy in NPG, for single state MPD it is easy to see that
\[
\pi_K(a)=\frac{\exp(\eta K r(a))}{\sum_b\exp(\eta K r(b))}.
\]
Hence, we have
\begin{align*}
    V^\star-V^{\pi_K}&=\frac{1}{1-\gamma}\sum_a\pi_K(a)(r^*-r(a))\\
    &=\frac{1}{1-\gamma}\frac{\sum_a(r^\star-r(a))\exp(\eta K r(a))}{\sum_a \exp(\eta K r(a))}\\
    & = \frac{1}{1-\gamma}\frac{\sum_{a\notin \mathcal{A}^\star}(r^\star-r(a))\exp(\eta K r(a))}{\sum_a \exp(\eta K r(a))}\\
    &\geq \frac{1}{1-\gamma}\frac{\Delta e^{\eta K (r^\star-\Delta)}}{\sum_a \exp(\eta K r(a))}\\
    &\geq \frac{1}{1-\gamma}\frac{\Delta e^{\eta K (r^\star-\Delta)}}{|\mathcal{A}| \exp(\eta K r^\star)}\\
    &=\frac{1}{1-\gamma}\frac{\Delta}{|\mathcal{A}|}e^{-\eta K \Delta},
\end{align*}
which concludes the result.

\end{proof}

\section{Details of Proof of Theorem \ref{thm: npg_adaptive}} \label{sec:Thm_3_proof}
First we provide the proof of Lemma \ref{lemma: npg_pi}. 
\begin{proof}[Proof of Lemma \ref{lemma: npg_pi}]
Using the performance difference Lemma \ref{lem:PDL_main}, we have
\begin{align}
\label{eq: v_bound}
|  V^{\pi_{k+1}} (\rho) - V^{\Tilde{\pi}_{k+1}}(\rho) | &=\bigg|\mathbb{E}_{s \sim d^{\pi_{k+1}}_\rho}\left[\mathbb{E}_{a \sim \pi_{k+1}(\cdot|s)}\left[A^{\Tilde{\pi}_{k+1}}(s,a)\right]\right]\bigg|\nonumber \\
    & \leq \sup_{s \in \mathcal{S}} \bigg| \sum_{a \in \mathcal{A}} \pi_{k+1}(a|s)A^{\Tilde{\pi}_{k+1}}(s,a) \bigg|\nonumber\\
    & \stackrel{(a)}{\leq } \sup_{s \in \mathcal{S}} \bigg| \sum_{a \in \mathcal{A}} (\pi_{k+1}(a|s)-\Tilde{\pi}_{k+1}(a|s))A^{\Tilde{\pi}_{k+1}}(s,a) \bigg|\nonumber\\
    & \stackrel{(b)}{\leq }\frac{|\mathcal{A}|}{1-\gamma}\sup_{s \in \mathcal{S}, a \in \mathcal{A}} \big|\pi_{k+1}(a|s)-\Tilde{\pi}_{k+1}(a|s) \big|,
\end{align}
where (a) holds since for any policy $\pi$ and any state $s\in \mathcal{S}$, $\sum_{a \in \mathcal{A}} \pi(a|s)A^\pi(s,a) = 0$, and (b) is due to Lemma \ref{lem:monotone}

The adaptive NPG update of policy $\pi_{k+1}$ is given by,
\begin{align*}
    \pi_{k+1} (a|s) = \frac{\pi_k(a|s) \exp(\eta_k Q^{\pi_k}(s,a)) }{\sum_{a'\in \mathcal{A}} \pi_k(a'|s) \exp(\eta_k Q^{\pi_k}(s,a')) }.
\end{align*}
Next, we will now consider two cases. First case is when $a\notin 
\mathcal{A}_s^{\pi_{k}} $. Under this case, we have $\Tilde{\pi}_{k+1}(a|s) =0$. Hence, for $a\notin \mathcal{A}_s^{\pi_{k}}$, we have
\begin{align}
\label{eq: case1}
    \big|\pi_{k+1}(a|s)-\Tilde{\pi}_{k+1}(a|s) \big| &= \frac{\pi_k(a|s) \exp(\eta_k Q^{\pi_k}(s,a)) }{\sum_{a'\in \mathcal{A}} \pi_k(a'|s) \exp(\eta_k Q^{\pi_k}(s,a')) }\nonumber\\
    & \leq \frac{\pi_k(a|s) \exp(\eta_k Q^{\pi_k}(s,a)) }{\sum_{a'\in \mathcal{A}_s^{\pi_{k}}} \pi_k(a'|s) \exp(\eta_k Q^{\pi_k}(s,a')) }\nonumber\\
    & = \frac{\pi_k(a|s)}{\sum_{a'\in \mathcal{A}_s^{\pi_{k}}} \pi_k(a'|s) } \times \exp \left( \eta_k (Q^{\pi_k}(s,a) - \max_{a'\in \mathcal{A}} Q^{\pi_k}(s,a')) \right)\nonumber\\
    &\leq \frac{\pi_k(a|s)}{\sum_{a'\in \mathcal{A}_s^{\pi_{k}}}\pi_k(a'|s) } \exp\left(- \eta_k \Delta^{\pi_k}(s)\right)\nonumber\\
    &\stackrel{(a)}{\leq} \max_{a'\in \mathcal{A}_s^{\pi_{k}}}\left\{ \frac{1}{\pi_{k}(a'|s)} \right\} \exp\left(- \eta_k \Delta^{\pi_k}(s)\right)\nonumber\\
    & \stackrel{(b)}{\leq} \frac{1}{|\mathcal{A}|} \exp(-L_k),
\end{align}
where (a) follows by using $\pi_k(s,a)\leq 1$ and (b) follows by using Eq. (\ref{eq: adaptive_step_size}).

Next we consider the second case, that is when $a\in \mathcal{A}^{\pi_k}_s$. For this case, $Q^{\pi_k}(s,a) = \max_{a'\in \mathcal{A} }Q^{\pi_k}(s,a')$, which gives us
\begin{align*}
    \sum_{a'\in \mathcal{A}}  \pi_k(a'|s) \exp \left(\eta_k (Q^{\pi_k}(s,a') - \max_{a''\in \mathcal{A} }Q^{\pi_k}(s,a''))\right) &\leq  \sum_{a'\in \mathcal{A}^{\pi_k}_s} \pi_k(a'|s) + \sum_{a'\notin \mathcal{A}^{\pi_k}_s} \pi_k(a'|s) \exp \left(-\eta_k\Delta^{\pi_k}(s) \right)\\
    & \leq \sum_{a'\in \mathcal{A}^{\pi_k}_s} \pi_k(a'|s) + \frac{1}{|\mathcal{A}|} \exp(-L_k) \min_{s'\in \mathcal{S}, a'' \in \mathcal{A}^{\pi_k}_{s'}} \pi_k(a''|s'),
\end{align*}
where the last inequality follows by using Eq. \ref{eq: adaptive_step_size} and the fact that $\sum_{a'\notin \mathcal{A}^{\pi_k}_s} \pi_k(a'|s) \leq 1$. By applying the above inequality, for any $a \in \mathcal{A}^{\pi_k}_s$ we get
\begin{align*}
    \pi_{k+1}  (a|s)&= \frac{\pi_k(a|s) \exp(\eta_k Q^{\pi_k}(s,a)) }{\sum_{a'\in \mathcal{A}} \pi_k(a'|s) \exp(\eta_k Q^{\pi_k}(s,a')) }\\
    & = \frac{\pi_k(a|s) }{\sum_{a'\in \mathcal{A}} \pi_k(a'|s) \exp(\eta_k (Q^{\pi_k}(s,a') - \displaystyle \max_{a''\in \mathcal{A} }Q^{\pi_k}(s,a'')))}\\
    & \geq  \frac{\pi_k(a|s)}{\sum_{a'\in \mathcal{A}^{\pi_k}_s} \pi_k(a'|s) + \exp(-L_k) \displaystyle \min_{s'\in \mathcal{S}, a'' \in \mathcal{A}^{\pi_k}_{s'}} \pi_k(a''|s')}\\
    & \geq \frac{1}{\frac{\sum_{a'\in \mathcal{A}^{\pi_k}_s} \pi_k(a'|s)}{\pi_k(a|s) } + \exp(-L_k)}\\
    & = \frac{\Tilde{\pi}_{k+1} (a|s)}{1+\Tilde{\pi}_{k+1} (a|s)\exp(-L_k) },
\end{align*}
where the last equality holds by the definition of $\Tilde{\pi}_{k+1} (a|s)$ in \eqref{eq: PI_update}. In addition, 
\begin{align*}
     \sum_{a'\in  \mathcal{A}}  \pi_k(a'|s) \exp \left(\eta_k (Q^{\pi_k}(s,a') - \max_{a''\in \mathcal{A} }Q^{\pi_k}(s,a''))\right)& \geq  \hspace{-2mm}\sum_{a'\in \mathcal{A}_s^{\pi_k}}\hspace{-2mm}  \pi_k(a'|s) \exp \left(\eta_k (Q^{\pi_k}(s,a') - \max_{a''\in \mathcal{A} }Q^{\pi_k}(s,a''))\right)\\
    & = \sum_{a'\in \mathcal{A}_s^{\pi_k}}  \pi_k(a'|s).
\end{align*}
This shows that for $a \in \mathcal{A}_s^{\pi_k}$
\begin{align*}
    \pi_{k+1}  (a|s)&= \frac{\pi_k(a|s) \exp(\eta_k Q^{\pi_k}(s,a)) }{\sum_{a'\in \mathcal{A}} \pi_k(a'|s) \exp(\eta_k Q^{\pi_k}(s,a')) }\\
    & = \frac{\pi_k(a|s) }{\sum_{a'\in \mathcal{A}} \pi_k(a'|s) \exp(\eta_k (Q^{\pi_k}(s,a') - \displaystyle \max_{a''\in \mathcal{A} }Q^{\pi_k}(s,a'')))}\\
    & \leq \frac{\pi_k(a|s)}{\sum_{a'\in \mathcal{A}_s^{\pi_k}}  \pi_k(a'|s)}\\
    & = \Tilde{\pi}_{k+1}(a|s).
\end{align*}

Thus, for $a \in \mathcal{A}^{\pi_k}_s$ we have 
\begin{align}
\label{eq: case2}
     \big|\pi_{k+1}(a|s)-\Tilde{\pi}_{k+1}(a|s) \big| & = \Tilde{\pi}_{k+1}(a|s) - \pi_{k+1}(a|s)\nonumber \\
     &\leq \Tilde{\pi}_{k+1}(a|s) - \frac{|\mathcal{A}|\Tilde{\pi}_{k+1} (a|s)}{|\mathcal{A}|+\Tilde{\pi}_{k+1} (a|s)\exp(-L_k) }\nonumber\\
     & = \frac{\big(\Tilde{\pi}_{k+1}(a|s)\big)^2 \exp(-L_k)}{|\mathcal{A}|+\Tilde{\pi}_{k+1} (a|s)\exp(-L_k) }\nonumber\\
     &\leq \frac{1}{|\mathcal{A}|} \big(\Tilde{\pi}_{k+1}(a|s)\big)^2 \exp(-L_k)\nonumber\\
     & \leq \frac{1}{|\mathcal{A}|} \exp(-L_k),
\end{align}
where the last inequality follows due to $\Tilde{\pi}_{k+1}(s,a) \leq 1$. By applying the results in Eq. (\ref{eq: case1}) and Eq. (\ref{eq: case2}) in Eq. (\ref{eq: v_bound}), we have
\begin{equation*}
    |  V^{\pi_{k+1}} (\rho) - V^{\Tilde{\pi}_{k+1}}(\rho) | \leq \frac{1}{1-\gamma} \exp(-L_k).
\end{equation*}
This completes the proof of Lemma \ref{lemma: npg_pi}.
\end{proof}

\begin{proof}[Proof of recursion in Theorem \ref{thm: npg_adaptive}]
\begin{align*}
     \sum_{i=0}^k  \gamma^i \exp\{-L(k-i)\} &= \exp\{-Lk\} \sum_{i=0}^k \exp\{(L+\log \gamma)i\} \\
     &=\exp\{-Lk\} \frac{\exp\{(L+\log \gamma)(k+1)\}-1}{\exp\{L+\log \gamma\}-1} \\
     &\leq \begin{cases}
     \frac{\exp\{L\}}{\exp\{L+\log \gamma\}-1}\gamma^{k+1} &\textit{if } L+\log \gamma >0, \\
     \frac{1}{1-\exp\{L+\log \gamma\}}\exp\{-Lk\} &\textit{if } L+\log \gamma <0
     \end{cases} \\
     &\leq \frac{1}{|\gamma - \exp\{-L\}|}\max\left\{\gamma^{k+1},\exp\{-L(k+1)\}\right\}.
\end{align*}
Finally, when $L=-\log(\gamma)$, 
\begin{align*}
     \sum_{i=0}^k  \gamma^i \exp\{-L(k-i)\}&=\sum_{i=0}^k \exp\{-Lk\}\\
     &=(k+1)\exp\{L\}\exp\{-L(k+1)\}
\end{align*}
\end{proof}
\section{Quadratic convergence of NPG} \label{app: quadratic_convergence}

\begin{lemma} \label{lemma: recursion_asymptotic}
For some $p \in (1,2]$, $b>0$, and non-negative sequences $\{r_k\}_{k \in \mathbb{Z}_+}$, and $\{e_k\}_{k \in \mathbb{Z}_+}$, assume
\begin{align*}
    r_k \leq e_{k-1}+b r_{k-1}^p \quad \forall k \in \mathbb{Z}_+.
\end{align*}
Then, we have 
\begin{align}
     r_k \leq \sum_{i=0}^{k-1} e_{k-1-i}^{p^i} b^{\frac{p^i-1}{p-1}} p^{\frac{p^{i+1}-(i+1)p+i}{(p-1)^2}}+p^{\frac{p^{k+1}-(k+1)p+k}{(p-1)^2}}b^{\frac{p^k-1}{p-1}}r_0^{p^k} \quad \forall k \in \mathbb{Z}_+. \label{eq: recursion}
\end{align}
\end{lemma}
\begin{proof}[Proof of Lemma \ref{lemma: recursion_asymptotic}]
We will prove this lemma by induction. For $k=1$, the RHS of \eqref{eq: recursion} reduces to 
\begin{align*}
    e_0+bpr_0^p \geq e_0+br_0^p \geq r_1,
\end{align*}
where the first inequality holds as $p \in (1,2]$ and $b,r_0>0$ and the second inequality holds by the hypothesis of the lemma. Thus, the base case is satisfied. Now, assume that \eqref{eq: recursion} is true for $k$ and we will now show that it holds for $k+1$.
\begin{align*}
    r_{k+1} &\leq e_k+br_k^p \\
    &\overset{(a)}{\leq} e_k+b\left(\sum_{i=0}^{k-1} e_{k-1-i}^{p^i} b^{\frac{p^i-1}{p-1}} p^{\frac{p^{i+1}-(i+1)p+i}{(p-1)^2}}+p^{\frac{p^{k+1}-(k+1)p+k}{(p-1)^2}}b^{\frac{p^k-1}{p-1}}r_0^{p^k}\right)^p \\
    &\overset{(b)}{\leq} e_k+b\sum_{i=0}^{k-1} p^{i+1} \left(e_{k-1-i}^{p^i} b^{\frac{p^i-1}{p-1}} p^{\frac{p^{i+1}-(i+1)p+i}{(p-1)^2}}\right)^p+bp^k\left(p^{\frac{p^{k+1}-(k+1)p+k}{(p-1)^2}}b^{\frac{p^k-1}{p-1}}r_0^{p^{k}}\right)^p \\
   &\leq e_k+\sum_{i=0}^{k-1} e_{k-1-i}^{p^{i+1}} b^{\frac{p^{i+1}-1}{p-1}} p^{\frac{p^{i+2}-(i+2)p+i+1}{(p-1)^2}}+p^{\frac{p^{k+2}-(k+2)p+k+1}{(p-1)^2}}b^{\frac{p^{k+1}-1}{p-1}}r_0^{p^{k+1}} \\
   & = e_k+\sum_{i=1}^{k} e_{k-i}^{p^{i}} b^{\frac{p^{i}-1}{p-1}} p^{\frac{p^{i+1}-(i+1)p+i}{(p-1)^2}}+p^{\frac{p^{k+2}-(k+2)p+k+1}{(p-1)^2}}b^{\frac{p^{k+1}-1}{p-1}}r_0^{p^{k+1}}\\
   & = \sum_{i=0}^{k} e_{k-i}^{p^{i}} b^{\frac{p^{i}-1}{p-1}} p^{\frac{p^{i+1}-(i+1)p+i}{(p-1)^2}}+p^{\frac{p^{k+2}-(k+2)p+k+1}{(p-1)^2}}b^{\frac{p^{k+1}-1}{p-1}}r_0^{p^{k+1}}
\end{align*}
where $(a)$ follows by the induction hypothesis. Next, $(b)$ follows because any $a,b \in \mathbb{R}_+$ and $p \in (1,2]$, $(a+b)^p \leq p(a^p+b^p)$, which by repeated application, extends to $(\sum_{i=1}^k a_i)^p \leq \sum_{i=1}^{k-1} p^i a_i^k+p^{k-1}a_k^p$. This completes the proof.
\end{proof}
\begin{lemma} Assume that $e_k=a\exp(-Lp^k)$ for all $k \in \mathbb{Z}_+$ for some $a,L>0$, and $\alpha>p$. Then we have
\begin{align*}
    \sum_{i=1}^{k} e_{k-i}^{p^i} b^{\frac{p^i-1}{p-1}} p^{\frac{p^{i+1}-(i+1)p+i}{(p-1)^2}} \leq k\exp\left\{p^{k}\left(\frac{\log b}{p-1}+\frac{p \log p}{(p-1)^2}+\log a-L\right)\right\}
\end{align*}
\end{lemma}
\begin{proof}
\begin{align*}
    \lefteqn{\sum_{i=1}^{k} e_{k-i}^{p^i} b^{\frac{p^i-1}{p-1}} p^{\frac{p^{i+1}-(i+1)p+i}{(p-1)^2}}}\\
    &=\sum_{i=1}^{k} a^{p^i}e^{-Lp^k } b^{\frac{p^i-1}{p-1}} p^{\frac{p^{i+1}-(i+1)p+i}{(p-1)^2}} \\
    &=\sum_{i=1}^{k} \exp\left\{-L p^k+p^i\left(\frac{\log b}{p-1}+\frac{p \log p}{(p-1)^2}+\log a\right)-i\frac{\log p}{p-1}-\frac{\log b}{p-1}-\frac{p}{(p-1)^2}\right\} \\
    &\leq \sum_{i=1}^{k} \exp\left\{-L p^k +p^i\left(\frac{\log b}{p-1}+\frac{p \log p}{(p-1)^2}+\log a\right)\right\} \\
    &\leq \exp\left\{p^{k}\left(\frac{\log b}{p-1}+\frac{p \log p}{(p-1)^2}+\log a \right)\right\}\sum_{i=1}^{k}\exp\left\{-Lp^k \right\} \\
    &=k\exp\left\{p^{k}\left(\frac{\log b}{p-1}+\frac{p \log p}{(p-1)^2}+\log a -L\right)\right\} 
\end{align*}
\end{proof}
\end{document}